\documentclass[twoside]{article}
\usepackage{comment, graphicx}
\usepackage{amsmath, amsthm, amssymb, amstext, comment, fullpage, natbib}
\author{Margareta Ackerman and Jarrod Moore}
\title{When is Clustering Perturbation Robust?}
\date{}
\newtheorem{definition}{Definition}

\newtheorem{theorem}{Theorem}
\newtheorem{lemma}{Lemma}

\begin{document}
\maketitle


\begin{abstract}
Clustering is a fundamental data mining tool that aims to divide data into groups of similar items.  Generally, intuition about clustering reflects the ideal case -- exact data sets endowed with flawless dissimilarity between individual instances. 

In practice however, these cases are in the minority, and clustering applications are typically characterized by noisy data sets with approximate pairwise dissimilarities. As such, the efficacy of clustering methods in practical applications necessitates robustness to perturbations. 

In this paper, we perform a formal analysis of perturbation robustness, revealing that the extent to which algorithms can exhibit this desirable characteristic is inherently limited, and identifying the types of structures that allow popular clustering paradigms to discover meaningful clusters in spite of faulty data. 
\end{abstract}

\begin{comment}
TO DO
\begin{itemize}
\item (Maya) Stress that our definition of perturbation robustness is based on how it was defined in previous work
\item (Maya) Our impossibility result holds for switching 2/3 of the data!!!
\item ! Prove that on a random reclustering we expect about 1/2 the points to switch. 
\item ! Change proof of positive results back to original approach. 
\item Include the Euclidean space proof for additive perturbation robustness. Stress the additive results more. Explain why we focus on multiplication. Relationships between additive and multiplicative. 
\item Include contradiction in Wright's axioms
\end{itemize}
\end{comment}
\section{Introduction}

Clustering is a popular data mining tool, due in no small part to its general and intuitive goal of dividing data into groups of similar items. Yet in spite of this seemingly simple task, successful application of clustering techniques in practice is oftentimes challenging. In particular, there are inherent difficulties in the data collection process and design of pairwise dissimilarity measures, both of which significantly impact the behavior of clustering algorithms.

%Clustering is a popular data mining tool, which aims to divide data into groups of similar items. Yet, despite the seeming simplicity of this task, applying clustering in practice is often challenging. The challenge of effectively applying this data mining tool often starts before a clustering method is chosen, as the data collection process and design of pairwise distances are crucial to the final outcome. 

Intuition about clustering often reflects the ideal case -- flawless data sets with well-suited dissimilarity between individual instances. In practice however, these cases are rare. Errors are introduced into a data set for a wide variety of reasons; from precision of instruments (a student's ruler to the Large Hadron Collider alike have a set precision), to human error when data is user-reported (common in the social sciences). Additionally, the dissimilarity between pairwise instances is often based on heuristic measures, particularly when non-numeric attributes are present. Furthermore, the dynamic nature of prominent clustering applications (such as personalization for recommendation systems) implies that by the time the data has been clustered, it has already changed. 
%Indeed, applications for which perfectly accurate data can be attained are in the minority. Errors are introduced in a variety of ways, from precision of instruments (often occurring in the life sciences), to human error when data is user-reported (common in the social sciences). Another critical issue is that distances between pairwise instances are often based on heuristic measures, particularly when non-numeric attributes are present. Furthermore, the dynamic nature of prominent clustering applications (such as personalization for recommendation systems) implies that by the time an algorithm has clustered the data, it has already changed. 

The ubiquity of flawed input poses a serious challenge. If clustering is to operate strictly under the assumption of ideal data, its applicability would be reduced to fairly rare applications where such data can be attained. As such, it would be desirable for clustering algorithms to provide some qualitative guarantees about their output when partitioning noisy data. This leads us to explore whether there are any algorithms for which such guarantees can be provided.

Although data can be faulty in a variety of ways, our focus here is on inaccuracies of pairwise distances. At a minimum, small perturbation to data should not radically affect the output of an algorithm. It would be natural to expect that some clustering techniques are more robust than others, allowing users to rely on perturbation robust techniques when pairwise distances are inexact.

%Although data can be faulty is a variety of ways, our focus here is on inaccuracies of pairwise distances. As such, a minimum requirement is that small perturbations to data would not radically affect the resulting clustering. It would be natural to expect that some clustering techniques are more robust than others, allowing users to rely on perturbation robust techniques when pairwise distance are inaccurate.

However, our investigation reveals that no reasonable clustering algorithm exhibits this desirable characteristic. In fact, both additive and multiplicative perturbation robustness are unrealistic requirements. We show that no clustering algorithm can satisfy robustness to perturbation without violating even more fundamental requirements. Not only do existing methods lack this desirable characteristic, but our findings also preclude the possibility of designing novel perturbation robust clustering methods. 

Perhaps it is already surprising that no reasonable clustering algorithm can be perfectly perturbation robust, but our results go further. Instead of requiring that the clustering remain unchanged following a perturbation, we allow up to \emph{two-thirds} of all pairwise distances to change (from in-cluster to between-cluster, or vice-versa). It turns our that this substantial relaxation doesn't overcome our impossibility theorem.

%Even if perturbation robustness is relaxed to preserve only one-third of all pairwise dissimilarities, it still cannot be attained by reasonable clustering technique.

%However, our investigation uncovered a major hindrance: It turns out that there are no reasonable clustering algorithms that exhibit this desirable characteristic. In fact, both additive and multiplicative perturbation robustness are unrealistic requirements. We show that no clustering algorithm can satisfy robustness to perturbation without violating even more fundamental requirements. This leads to the pessimistic conclusion that clustering algorithms are not equipped to handle noisy data. 

Luckily, further exploration paints a more optimistic picture. A careful examination of this issue requires a look back to the underlying goal of clustering, which is to discover clustering structure in data \emph{when such structure is present}. Our investigation suggests that sensitivity to small perturbations is inevitable only on unclusterable instances, for which clustering is inherently ill-suited. As such, it can be argued that whether an algorithm exhibits robustness on such data is inconsequential. 

On the other hand, we show that when data is endowed with inherent structure, existing methods can often successfully reveal that structure even on faulty (perturbed) data. We investigate the type of cluster structures required for the success of popular clustering techniques, showing that the robustness of $k$-means and related methods is directly proportional to the degree of inherent cluster structure. Similarly, we show that popular linkage-based techniques are robust when clusters are well-separated. Furthermore, different cluster structures are necessary for different algorithms to exhibit robustness to perturbations.

\subsection{Previous work}

This work follows a line of research on theoretical foundations of clustering. 
Efforts in the field began as early as the 1970s with the pioneering work of Wright~\cite{Wright} on axioms of clustering, as well analysis of clustering properties  by Fisher et al~\cite{fisher1971admissible} and Jardine et al~\cite{Jardine}, among others. This field saw a renewed surge of activity following Kleinberg's~\cite{Kleinberg} famous impossibility theorem, when he showed that no clustering function can simultaneously satisfy three simple properties. Also related to our work is a framework for selecting clustering methods based on differences in their input-output behavior~\cite{NIPS2010,ackerman2011weighted,Jardine,Reza,COLT2010,AISTATS2013}
as well as research on clusterability, which aims to quantify the degree of inherent cluster structure in data~\cite{ben2015computational,AISTATS2009,Balcan,blum,Ostrovsky}.

%Later, Ackerman and Ben-David~\cite{NIPS2008} showed that one of these axioms (``Consistency'') has counter-intuitive consequences.

Previous work on perturbation robustness studies it from a computational perspective by identifying new efficient algorithms for robust instances \cite{Bilu,AISTATS2009,Awasthi}. Ben-David and Reyzin~\cite{ben2014data} recently studied corresponding NP-hardness lower bounds. 

In this paper, we take a fresh look at perturbation robustness. We begin our investigation by asking when perturbation robustness is possible to attain. After proving that robustness to perturbations cannot be achieved as a data-independent property of an algorithm, we seek to understand when popular clustering paradigms satisfy this requirement. Our analysis of established methods is an essential complement to efforts in algorithmic development, as the need for understanding established methods is amplified by the fact that most clustering users rely on a small number of well-known techniques. Our results demonstrate the type of cluster structures required for robustness of popular clustering paradigms. 

\section{Definitions and notation}

Clustering is a wide and heterogeneous domain. For most of this paper, we focus on a basic sub-domain where the input to a clustering function is a finite set of points endowed with a between-points  dissimilarity function, and the number of clusters ($k$), and the output is a partition of that domain. 

%It is worth noting that all of our results also hold in Euclidean space. 

A \emph{dissimilarity function} is a symmetric function $d: X \times X \rightarrow
R^+$, such that $d(x,x) = 0$ for all $x \in X$. The data sets that we consider are pairs $(X,d)$, where $X$ is some finite domain set and $d$ is a dissimilarity function over $X$.

A \emph{$k$-clustering} $C = \{C_1, C_2, \ldots, C_k\}$ of a data set $X$ is a
partition of $X$ into $k$ disjoint subsets (or, clusters) of $X$ (so, $\displaystyle \bigcup_{i} C_i =
X$). A \emph{clustering} of $X$ is a $k$-clustering of $X$ for some $1 \leq k
\leq |X|$.

For a clustering $C$, let $|C|$ denote the number of clusters in $C$ and $|C_i|$ denote the number of points in a cluster $C_i$. For a domain $X$, $|X|$ denotes the number of points in $X$, which we denote by $n$ when the domain is clear from context. We write $x \sim_{\mathcal{C}} y$ if $x$ and $y$ are both in some cluster $C_j$; and $x \not\sim_{\mathcal{C}} y$ otherwise. This is an equivalence relation. 

The \emph{Hamming distance} between clusterings $C$ and $C'$ of the same domain set $X$ is defined by 
\[
\Delta(C, C') = \frac{ |\{\{x,y\} \subset X \mid (x \sim_C y) \oplus (x \sim_{C'} y )\}|}{\binom{|X|}{2}},\]
where $\oplus$ denotes the logical XOR operation.  

That is, the difference is the number of edges that disagree, being in-cluster in one of the clusterings and between-cluster in the other. The maximum distance between clusterings is when the Hamming distance is $1$.

%\begin{definition}[Isomorphisms]
%A few notions of isomorphisms of structures are relevant to our discussion. 

%\begin{enumerate}
%\item We say that $(X,d)$ and  $(X',d')$ are \emph{isomorphic domains}, denoted $(X,d) \sim (X',d')$, if there exists a bijection $\phi: X \rightarrow X'$ so that $d(x,y) = d'(\phi(x),\phi(y))$ for all $x,y \in X$.

%\item We say that two clusterings  $C$ of some domain $(X,d)$ and $C'$ of some domain $(X', d')$ are \emph{isomorphic clusterings}, denoted $(C,d) \cong (C',d')$, if there exists a domain isomorphism
%$\phi : X \to X'$ so that $x$ and $y$ belong to the same cluster in $C$
%if and only if  $\phi(x)$ and $\phi(y)$ share a cluster in $C'$.
%\end{enumerate}
%\end{definition}

%We now define clustering functions, requiring that they be isomorphism invariant, namely, that the output should not depend on data labels. That is, we require that algorithms return the same results on isomorphic data sets.\footnote{This elementary requirement has been proposed as an axiom of clustering by  (\cite{NIPS2008}), and is required of clustering functions in (\cite{COLT2010}) and (\cite{ackerman2011discerning})} 

%\begin{definition}[clustering function]\label{clustering function}
%A \emph{clustering function} is a function
%that takes as input a pair $(X,d)$  and outputs a clustering of the domain $X$.
%\end{definition}

%We also consider clustering functions that require the number of clusters to be part of the input, as is standard for many clustering techniques. 

Lastly, we formally define clustering functions. 

\begin{definition}[Clustering function]\label{clustering function}
A \emph{clustering function} is a function $\mathcal{F}$
that takes as input a pair $(X,d)$ and a parameter $1 \leq k \leq |X|$, and outputs a $k$-clustering of the domain $X$.
%such that whenever $(X,d) \sim (X',d')$, then $\mathcal{F}(X,d) \cong \mathcal{F}(X',d')$. 
\end{definition}

\section{Perturbation robustness as a property of an algorithm}\label{impossibility}

Whenever a user is faced with the task of clustering faulty data, it would be natural to select an algorithm that is robust to perturbations of pairwise dissimilarities. As such, we begin our study of perturbation robustness by casting it as a property of an algorithm. If we could classify algorithms based on whether or not (or to what degree) they are perturbation robust, then clustering users could incorporate this information when making decisions regarding which algorithms to apply on their data. 

First, we define what it means to perturb a dissimilarity function.

\begin{definition}[$\alpha$-multiplicative perturbation of a dissimilarity function]\label{multiplicativeperturbationdissimilarity}
Given a pair of dissimilarity functions $d$ and $d'$ over a domain $X$, $d'$ is an \emph{$\alpha$-multiplicative-perturbation} of $d$, for $\alpha>1$, if for all $x,y \in X$, $\alpha ^{-1} d(x,y)\leq d'(x,y) \leq \alpha d(x,y)$.
\end{definition}

Additive perturbation of a dissimilarity function is defined analogously. 

\begin{definition}[$\epsilon$-additive perturbation of a dissimilarity function]\label{additiveperturbationdissimilarity}
Given a pair of dissimilarity functions $d$ and $d'$ over a domain $X$, $d'$ is an \emph{$\epsilon$-additive perturbation} of $d$, for $\epsilon>0$, if for all $x,y \in X$, $d(x,y)-\epsilon \leq d'(x,y) \leq d(x,y)+\epsilon$.
\end{definition}

It is important to note that all of our results hold for both multiplicative and additive perturbation robustness. Perturbation robust algorithms should be invariant to data perturbations; that is, if data is perturbed, then the output of the algorithm shouldn't change. This view of perturbation robustness is not only intuitive, but is also based on previous formulations \cite{Reyzin,Bilu,Awasthi} (This can be formalized as a property of clustering functions by setting $\delta = 0$ in Definition~\ref{mult-function} below). 

%Due to space constraints, many of the proofs for the additive case are deferred to the Appendix. 

However, requiring that the partitioning be identical before and after perturbation is provably too strict a requirement for clustering algorithms, as it can only hold for functions that effectively ignore all pairwise distances. That is, this notion of perturbation robustness only holds for clustering functions that, given any domain set $X$ and integer $1 \leq k \leq |X|$, produce the same partitioning regardless of the setting of $d$ (See Section 1 of the Appendix for details). 

As such, we introduce a relaxation that allows some error in the output of the algorithm on perturbed data. From a practical point of view, it is likely that a user who has only a perturbation of the true data set is likely to be satisfied with an approximately correct solution.

\begin{definition}\label{mult-function}
A clustering function $\mathcal{F}$ is \emph{$(\alpha, \delta)$-multiplicative perturbation robust} if, given any data set $(X,d)$ and $1 \leq k \leq |X|$, whenever $d'$ is an $\alpha$-multiplicative perturbation of $d$,  $$\Delta(\mathcal{F}(X,d,k), \mathcal{F}(X,d',k))<\delta.$$
\end{definition}

Additive perturbation robustness is defined analogously, by replacing the $\alpha$-multiplicative perturbation of the dissimilarity function with an $\epsilon$-additive perturbation. 

%In previous work, perturbation robustness is often defined with $\delta=0$. Unfortunately, as a property of an algorithm this definition fails. See Appendix Section X for details. 

%We now investigate whether this relaxation enables algorithms to be perturbation robust. 

%\subsection{Impossibility theorem for clustering functions}\label{k-impossibility}

\begin{figure}
\begin{center}
\includegraphics[scale=.75]{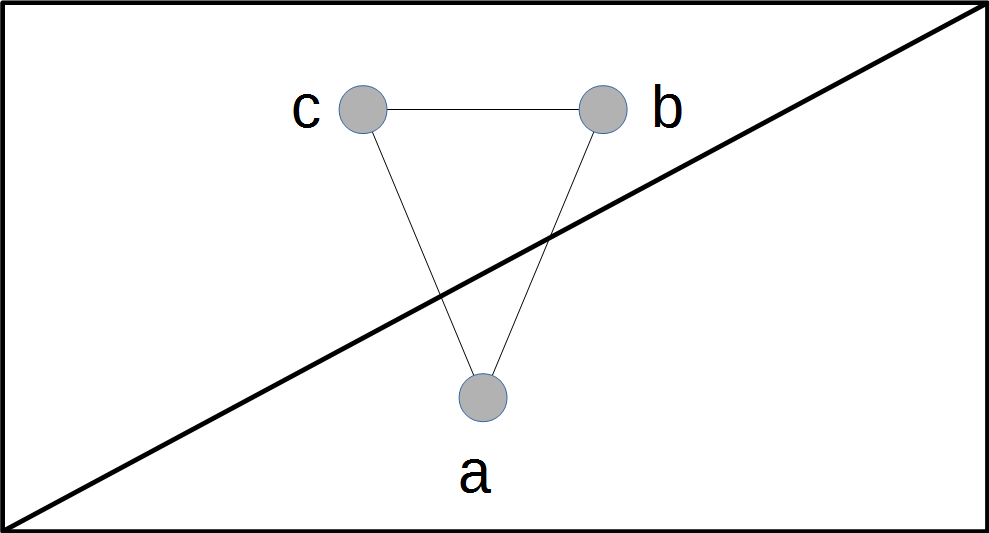}
\end{center}
\caption{An illustration of the three-body rule. Based on the main objective of clustering, which is to group similar items together, this rule requires that whenever an algorithm is given exactly three points and the number of clusters is two, then it groups the two closest elements.}
\label{fig:3body}
\end{figure}

%This section presents one of our main results, showing that even our relaxed notion of perturbation robustness conflicts with other, more , requirements of clustering functions.

Despite this substantial relaxation, perturbation robustness \emph{for $\delta$ as high as 2/3} is inherently incompatible with even more elementary requirements, as shown below. 

\subsection{Impossibility theorem for clustering functions}\label{k-impossibility}

We now proceed to show that perturbation robustness is too strong a requirement for clustering algorithms, and as such neither existing nor novel techniques can have this desirable characteristic. 

Particularly notable is that the impossibility results persists when $\delta$ is as high as $2/3$, meaning that a perturbation is allowed to change up to two-thirds of all pairwise distances from in-cluster to between-cluster, or vise-versa. As such, we show that no reasonable clustering algorithm can preserve more than a third of its pairwise distances after a perturbation. 

The following impossibility result derives from the pioneering work of Wright~\cite{Wright} on axioms of clustering. Wright originally proposed his axioms in Euclidean space, here we generalize them for arbitrary pairwise dissimilarities.

The first axiom we discuss follows from Wright's 11th axiom, and captures the very essence of clustering: to group similar items. This property considers an elementary scenario, requiring that given exactly three points, an algorithm asked for two clusters should group the two closest elements. See Figure \ref{fig:3body} for an illustration. A special case of this rule occurs when the three elements lie on the real line, in which case the furthest endpoint should be placed in its own cluster. 

\begin{definition}[Three-body rule]
Given a data set $X=\{a,b,c\}$, if $d(a,b)>d(b,c)$ and $d(a,c)>d(b,c)$, then $\mathcal{F}(X,d,2) = \{\{a\}, \{b,c\}\}$. 
\end{definition}

Wright's 6th axiom, and the final one we consider here, requires that replicating all data points by the same number should not change the clustering output. Outside of Euclidean space, we \emph{replicate} a point $x$ by adding a new element $x'$ and setting $d(x',y) = d(x,y)$ for all $y \in X$.

\begin{definition}[Replication invariance]
Given any positive integer $r$, if all points are replicated $r$ times,
then the partitioning of the original data is unchanged and all replicas lie in the same cluster as their original element. 
\end{definition}

Not only are these two axioms natural, as violating them leads to counterintuitive behavior, but they also hold for common techniques. It is easy to show that they are satisfied by common clustering paradigms, including cost-based methods such as $k$-means, $k$-median, and $k$-medoids, as well as linkage-based techniques, such as single-linkage, average-linkage and complete-linkage.

We now prove that no clustering function that satisfies the three-body rule and replication invariance can be perturbation robust. Furthermore, our result holds for all values of $\delta \leq 2/3$. Note that the following result applies to arbitrarily large data sets, for both multiplicative and additive perturbations.

\begin{theorem}
For any $\delta \leq 2/3$, $\alpha > 1$, and $\epsilon>0$, there is no clustering function that satisfies
\begin{enumerate}
\item  $(\alpha, \delta)$-multiplicative perturbation robustness, replication invariance, and the three-body rule, and
\item  $(\epsilon, \delta)$-additive perturbation robustness, replication invariance, and the three-body rule.

Further, the result holds for arbitrarily large data. 
\end{enumerate}
\end{theorem}
\begin{proof}
We proceed by contradiction, assuming that there exists a clustering function $\mathcal{F}$ that is replication invariant, adheres to the three-body rule, and is $(\alpha, \delta)$-multiplicative perturbation robust for some $\delta \leq 2/3$. 

Consider a data set $X=\{a, b, c\}$ with a distance function $d$ such that $d(b, c)<d(a, b)<d(a, c)$ and $d(a, b) = \alpha d(b, c)$. By the three-body rule,  $\mathcal{F}(X,d,2) = \{\{b, c\}, \{a\}\}$. We now replicate each point an arbitrary number of times, $r$, creating three sets $A, B, C$ such that all points that are replicas of the point $a$ and $a$ itself belong to $A$, all points that are replications of the point $b$ and $b$ itself belong to $B$, and similarly for $C$. By replication invariance, $\mathcal{F}(A\cup B\cup C,d,2) = \{B \cup C, A\}$.

Next, we apply an $\alpha$-multiplicative perturbation, creating distance function $d'$ such that $d'(a, b)<d'(b, c)<d'(a, c)$ and $d'(c, b) = \alpha \dot d'(b, a)$. By the three-body rule, $\mathcal{F}(A\cup B\cup C,d',2) = \{B \cup A, C\}$, and yet $(\alpha, 2/3)$-multiplicative perturbation robustness requires that the Hamming distance between $\mathcal{F}(A\cup B\cup C,d,2)$ and $\mathcal{F}(A\cup B\cup C,d',2)$  must be less than $2/3$. But as the Hamming distance between $\{B \cup C, A\}$ and $\{B \cup A, C\}$ is exactly $2/3$, we reach a contradiction. 

For additive perturbation, set $d$ so that  $d(b, c)<d(a, b)<d(a, c)$ and $d(a, b) =  d(b, c) + 0.5 \epsilon$. By the three-body rule,  $\mathcal{F}(X,d,2) = \{\{b, c\}, \{a\}\}$. As for the multiplicative case, we replicate each point $r$ times, creating three sets $A, B, C$. By replication invariance, $\mathcal{F}(A\cup B\cup C,d,2) = \{B \cup C, A\}$. We apply an $\epsilon$-additive perturbation to make distance function $d'$ such that $d'(a, b)<d'(b, c)<d'(a, c)$ and $d'(c, b) = d'(b, a) + 0.5 \epsilon$. By the three-body rule, $\mathcal{F}(A\cup B\cup C,d',2) = \{B \cup A, C\}$, and yet $(\alpha, 2/3)$-additive perturbation robustness requires that the Hamming distance between $\mathcal{F}(A\cup B\cup C,d,2)$ and $\mathcal{F}(A\cup B\cup C,d',2)$  must be less than $2/3$, reaching a contradiction. 
\end{proof}

%We have shown before that both multiplicative perturbation robustness and additive perturbation robustness are impossible in the general case when data is presented as an abstract dissimilarity matrix. 

Note that the above result holds if the data is in Euclidean space.  This allows us to view perturbations as small movements in space, required to satisfy certain constraints such as the triangle inequality as well as adhering to the dissimilarity constraints required by Definitions~\ref{multiplicativeperturbationdissimilarity} and \ref{additiveperturbationdissimilarity}. See supplementary material for details.

\section{Perturbation robustness as a property of data}

The above section demonstrates an inherent limitation of perturbation robustness as a property of clustering algorithms, showing that no reasonable clustering algorithm can exhibit this desirable characteristic. However, it turns out that perturbation robustness is possible to achieve when we restrict our attention to data endowed with inherent structure. 

As such, perturbation robustness becomes a property of both an algorithm and a specific data set. We introduce a definition of perturbation robustness that directly addresses the underlying data.  

\begin{definition}[$(\alpha, \delta)$-multiplicative perturbation robustness of data]
A data set $(X,d)$ satisfies \emph{$(\alpha, \delta)$-multiplicative perturbation robustness} with respect to clustering function $\mathcal{F}$ and $1 \leq k \leq |X|$, if for any $d'$ that is an $\alpha$-multiplicative perturbation of $d$, $$\Delta(\mathcal{F}(X,d,k), \mathcal{F}(X,d',k)) < \delta.$$
\end{definition}

Additive perturbation robustness of data is defined analogously. 

This perspective at perturbation robustness raises a natural question: On what types of data are algorithms perturbation robust? Next, we explore the type of structures that allow popular cost-based paradigms and linkage-based methods to uncover meaningful clusters even when data is faulty. 

%The answer turns out to vary by type of algorithm. Different clustering paradigms are robust to perturbation when the data has different cluster structures. That is, cluster structure for which some algorithms are perturbation robust, does not necessarily lead to robustness of other clustering methods.

\subsection{Perturbation robustness of $k$-means, $k$-medoids, and min-sum}

We begin our study of data-dependent perturbation robustness by considering cluster structures required for perturbation robustness of some of the most popular clustering functions: $k$-means, $k$-medoids and min-sum. 

Recall that $k$-means~\cite{steinley2006k} finds the clustering $C = \{C_1,\ldots,C_k\}$ that minimizes $\sum_{i=1} ^k \sum_{x \in C_i}d(x,c_i)^2,$ where $c_i$ is the center of mass of cluster $C_i$. An equivalent formulation that does not rely on centers of mass appears in \cite{Ostrovsky}. A closely related clustering function is $k$-medoids, where centers are required to be part of the data. Formally, the $k$-medoids cost of $C$ is $\sum_{i=1} ^k \sum_{x \in C_i}d(x,c_i),$ where $c_i \in C_i$ is chosen to minimize the objective. Lastly, the min-sum~\cite{sahni1976p} clustering function is the sum of all in-cluster distances, $\sum_{i=1} ^k \sum_{x, y \in C_i}d(x, y)$. 

Many different notions of clusterability have been proposed in prior work~\cite{AISTATS2009, ben2015computational}. Although they all aim to quantify the same tendency, it has been proven that notions of clusterability are often pairwise inconsistent~\cite{AISTATS2009}. As such, care must be taken when selecting amongst them.

In order to analyze $k$-means and related functions, we turn our attention to an intuitive cost-based notion, which requires that clusterings of near-optimal cost be structurally similar to the optimal solution. That is, this notion characterizes clusterable data as that which has a unique optimal solution in a strong sense, by excluding the possibility of having radically different clusterings of similar cost.  
See Figure~\ref{fig:UO} for an illustration.

This property, called ``uniqueness of optimum''\footnote{This notion of clusterability appeared under several different names. The term ``uniqueness of optimum'' was coined by Ben-David~\cite{ben2015computational}.} and closely related variations were investigated by \cite{Balcan}, \cite{Ostrovsky}, \cite{agarwal2013k} and \cite{AISTATS2013}, among others. See~\cite{Balcan} for a detailed exposition.  

\begin{definition}[Uniqueness of optimum]
Given a clustering function $\mathcal{F}$, a data set $(X,d)$ is \emph{$(\delta, c,c_0,k)$-uniquely optimal} if for every $k$-clustering $C$ of $X$ where $cost(C) \leq c \cdot cost(\mathcal{F}(X,d,k))+c_0$, 
$$\Delta(\mathcal{F}(X,d,k), C) < \delta.$$
\end{definition}

We show that whenever data satisfies the uniqueness of optimum notion of clusterability, $k$-means, $k$-medoids, and min-sum are perturbation robust. Furthermore, the degree of robustness depends on the extent to which the data is clusterable.

\begin{figure}
\begin{center}
\includegraphics[scale=0.32]{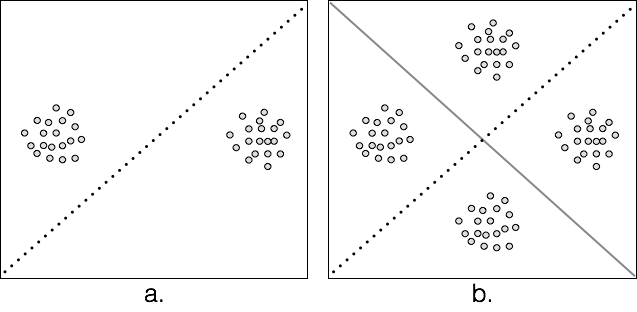}
\end{center}
\caption{An illustration of the uniqueness of optimum notion of clusterability for two clusters. Consider $k$-means, $k$-medoids, or min-sum. The highly-clusterable data depicted in (a) has a unique optimal solution, with no structurally different clusterings of near-optimal cost. In contrast, (b) displays data with two radically different clusterings of near-optimal cost, making this data poorly-clusterable for $k=2$.}
\label{fig:UO}
\end{figure}

For the following proofs we will use $cost_d(C)$ to denote the cost of clustering $C$ with the distance function $d$. We now show the relationships between uniqueness of optimum and perturbation robustness for $k$-means. 

%Likewise $cost_{d'}(C)$ is the cost of the clustering $C$ with the distance function $d'$.

\newpage

\begin{theorem}
Consider the $k$-means clustering function and a data set $(X,d)$. If $(X,d)$ is \emph{$(\delta, c, c_0, k)$-uniquely optimal}, then it is also $(\epsilon, \delta,k)$-additive perturbation robust for all $\epsilon <min(\frac{c-1}{2},\frac{- M + \sqrt{ M ^2 + 4 M C_0 }} {2 M})$, where $M = \binom{n}{2}$.
%Given the $k$-means objective, $(\delta, 1+2\epsilon, \binom{n}{2}\epsilon ^2 + 2\epsilon, k)$-uniquely optimal data sets are $(\epsilon, \delta)$-additive perturbation robust.
\end{theorem}
\begin{proof}
Consider a data set $(X, d)$, and let $d'$ be any $\epsilon$-additive perturbation of $d$. 

First, we argue that $cost_{d'}(\mathcal{F}(X,d,k))$ is close to $cost_{d'}(\mathcal{F}(X,d',k))$. Let $C = \mathcal{F}(X,d,k)$.
First, note that $cost_{d'}(\mathcal{F}(X,d',k)) \leq cost_{d'}(C)$. This is because $\mathcal{F}$ finds the optimal solution on $(X,d')$, and so the clustering it selects can only have lower or equal to cost than the cost of $C$ when evaluated with $d'$. 

So, we calculate the $k$-means cost of $C$ on $(X,d')$. The $k$-means objective function is equivalent to $\sum_{i=1}^k \frac{1}{|C_i|}\sum_{x,y \in C_i}d(x,y)^2$\cite{Ostrovsky}. After an additive perturbation, any pairwise distance, $d(x, y)$, is bounded by $d(x, y) + \epsilon$. In addition, the contribution of any in-cluster pairwise distance to the total cost of the clustering is proportional to the magnitude of the distance. It therefore follows

\begin{comment}
So, the $k$-means cost function is bounded by $\sum_{i=1}^k \sum_{x,y \in C_i}d(x,y)^2$~. After an additive perturbation, the contribution of an edge of length $x$, which used to contribute at most $x^2$ to the cost of the function, contributes at most $x^2 + 2\epsilon x + \epsilon ^2$, and so the contribution increases by at most $2 \epsilon x+\epsilon ^2$. It follows that
\end{comment}

%$cost(F(X,d')) \leq cost(C, (X,d')) \leq \sum_{x = (p_1,p_2) \in X} x^2 + 2\epsilon x + \epsilon ^2 < (1+2\epsilon )cost(F(X,d,k))+\binom{n}{2} \epsilon ^2.$

\begin{subequations}
\label{eq:leqList}
\begin{align}
 \label{eq:leqList:leq1} cost_{d'}(\mathcal{F}(X,d', k)) & \leq cost_{d'}(C)\\
\label{eq:leqList:leq2} &\leq \sum ^k _{i=1} \frac{1}{|C_i|} \sum _{\{x, y\} \subseteq C_i} (d(x, y) + \epsilon) ^2 \\
\label{eq:leqList:leq3} &\leq \sum _{i=1} ^k \sum _{\{x, y\} \subseteq C_i}\frac{1}{|C_i|} (d(x, y) + \epsilon)^2\\
\label{eq:leqList:leq4} &\leq \sum _{i=1} ^k \sum _{\{x, y\} \subseteq C_i} \frac{1}{|C_i|} \left[ d(x, y)^2 + 2d(x, y)\epsilon + \epsilon ^2 \right]
\end{align}
\end{subequations}
By distributing the summation in the inequality \ref{eq:leqList:leq4} we come to:
\begin{subequations}
\label{eq:sumExpansion}
\begin{align}
  \label{eq:sumExpansion:term1} cost_{d'}(\mathcal{F}(X, d', k))& \leq \sum _{i=1} ^k \sum _{\{x, y\} \subseteq C_i} \frac{1}{|C_i|} d(x, y)^2  \\
\label{eq:sumExpansion:term2} &+ \sum _{i=1} ^k \sum _{\{x, y\} \subseteq C_i} \frac{1}{|C_i|} 2d(x, y)\epsilon \\
\label{eq:sumExpansion:term3} &+  \sum _{i=1} ^k \sum _{\{x, y\} \subseteq C_i} \frac{1}{|C_i|}\epsilon ^2
\end{align}
\end{subequations}
The first term, \ref{eq:sumExpansion:term1}, is equivalent to $cost_d(\mathcal{F}(X, d, k))$. We deal with the second term, \ref{eq:sumExpansion:term2}, by defining two sets $S_1$ and $S_2$. To define $S_1$, we first define $S_{1i}$. $S_{1i} = \{ \{x, y\} \subseteq C_i | d(x, y) > 1\}$. Then $S_{1} = \{S_{1i} | 1\leq i \leq k\}$. Similarly $S_{2i} = \{ \{x, y\} \subseteq C_i | d(x, y) \leq 1\}$, and $S_2 = \{S_{2i} | 1\leq i \leq k\}$. 
\begin{subequations}
\label{eq:dxy1}
\begin{align}
\label{eq:dxy1:1} \sum _{i=1} ^k \sum _{\{x, y\} \subseteq C_i} \frac{1}{|C_i|} 2d(x, y)\epsilon & \leq \sum _{i=1} ^k \sum _{\{x, y\} \in S_{1i}} \frac{1}{|C_i|} 2d(x, y)\epsilon \\
\label{eq:dxy1:2} &+ \leq \sum _{i=1} ^k \sum _{\{x, y\} \in S_{2i}} \frac{1}{|C_i|} 2d(x, y)\epsilon
\end{align}
\end{subequations}
Because for all $\{ x, y \} \in S_{1i}$ for all $1 \leq i \leq k$, $d(x, y) > 1$, we can square the  $d(x, y)$ value in term \ref{eq:dxy1:1} while only increasing the total value. Likewise, we can replace the $d(x, y)$ value in term \ref{eq:dxy1:2}  with 1 while only increasing the total value.

\begin{subequations}
\label{eq:dxy2}
\begin{align}
\label{eq:dxy2:1} \sum _{i=1} ^k \sum _{\{x, y\} \subseteq C_i} \frac{1}{|C_i|} 2d(x, y)\epsilon & \leq   \sum _{i=1} ^k \sum _{\{x, y\} \in S_{1i}} \frac{1}{|C_i|} 2d(x, y)^2\epsilon \\
\label{eq:dxy2:2} & +  \sum _{i=1} ^k \sum _{\{x, y\} \in S_{2i}}  \frac{1}{|C_i|} 2\epsilon
\end{align}
\end{subequations}

Since $S_{1i}$ and $S_{2i}$ both consist of point pairs in $C_i$ and we are looking for an upper bound:
\begin{subequations}
\label{eq:dxy3}
\begin{align}
\label{eq:dxy3:1} \sum _{i=1} ^k \sum _{\{x, y\} \subseteq C_i} \frac{1}{|C_i|} 2d(x, y)\epsilon & \leq  \sum _{i=1} ^k \sum _{\{x, y\} \subseteq C_i} \frac{1}{|C_i|} 2d(x, y)^2\epsilon \\
\label{eq:dxy3:2} & + \sum _{i=1} ^k \sum _{\{x, y\} \subseteq C_i} \frac{1}{|C_i|} 2\epsilon
\end{align}
\end{subequations}
Note that $\sum _{i=1} ^k \sum _{\{x, y\} \subseteq C_i} \frac{1}{|C_i|} 2d(x, y)^2\epsilon$ is equivalent to $2\epsilon cost_d(\mathcal{F}(X, d)$.  
We can now return to the original inequality.\\

\begin{subequations}
\label{eq:sumExpansion3}
\begin{align}
\label{eq:sumExpansion3:1} cost_{d'}(\mathcal{F}(X, d', k)) & \leq cost_d(\mathcal{F}(X, d, k))  \\
\label{eq:sumExpansion3:2} & + 2\epsilon cost_d(\mathcal{F}(X, d, k)\\
 \label{eq:sumExpansion3:3} &+ \sum _{i=1} ^k \sum _{\{x, y\} \subseteq C_i} \frac{1}{|C_i|} 2\epsilon\\
\label{eq:sumExpansion3:4} & + \sum _{i=1} ^k \sum _{\{x, y\} \subseteq C_i} \frac{1}{|C_i|} \epsilon ^2
\end{align}
\end{subequations}
While we cannot know the size of individual clusters in the general case we do know $\frac{1}{|C_i|}$ is upper bounded by 1. Therefore we can substitute $\frac{1}{|C_i|}$ with 1 in terms \ref{eq:sumExpansion3:3} and \ref{eq:sumExpansion3:4} while only increasing the value. For the same reasons we do not know the number of in-cluster point pairs in the general case in terms \ref{eq:sumExpansion3:3} and \ref{eq:sumExpansion3:4}. However we do know the number of in-cluster point pairs is bounded by the total number of point pairs, namely $\binom{n}{2}$ which can be substituted in the same way while only increasing the value.
\begin{subequations}
\label{eq:sumExpansion4}
\begin{align}
\label{eq:sumExpansion4:1} cost_{d'}(\mathcal{F}(X, d', k)) & \leq cost_d(\mathcal{F}(X, d, k))  \\
\label{eq:sumExpansion4:2} & + 2\epsilon cost_d(\mathcal{F}(X, d, k))\\
 \label{eq:sumExpansion4:3} &+ \binom{n}{2} 2\epsilon\\
\label{eq:sumExpansion4:4} & + \binom{n}{2} \epsilon ^2
\end{align}
\end{subequations}

Therefore we know: $$cost_{d'}(\mathcal{F}(X, d', k)) \leq (1 + 2\epsilon)cost_d(\mathcal{F}(X, d, k)) + \binom{n}{2} (2\epsilon + \epsilon ^2)$$

Then, $c\geq 1+2 \epsilon $, so $\epsilon \leq \frac{c-1}{2}$. Similarly, $c_0 \geq M(\epsilon^2 + 2\epsilon)$,  so $\epsilon \leq \frac{- M + \sqrt{ M ^2 + 4 M C_0 }} {2 M}$ where $M = \binom{n}{2}$. So, $\epsilon <min(\frac{c-1}{2},\frac{- M + \sqrt{ M ^2 + 4 M C_0 }} {2 M})$. 
\end{proof}

\begin{theorem}
Consider the $k$-means clustering function and a data set $(X,d)$. If $(X,d)$ is \emph{$(\delta, c, c_0, k)$-uniquely optimal}, then it is also $(\alpha, \delta,k)$-multiplicative perturbation robust for all $\alpha<\sqrt{c}$.

%Given the $k$-means objective, a data set that is $(\delta, \alpha ^2, 0, k)$-uniquely optimal is also $(\alpha, \delta)$-multiplicative perturbation robust.
\end{theorem}

\begin{proof}
Consider a data set $(X,d)$, and let $d'$ be any $\alpha$-multiplicative perturbation of $d$. 

First, we argue that $cost_{d'}(\mathcal{F}(X,d, k))$ is close to $cost_{d'}(\mathcal{F}(X,d', k))$. Let $C = \mathcal{F}(X,d, k)$.
First, note that $cost_{d'}(\mathcal{F}(X,d', k)) \leq cost_{d'}(C)$. This is because $\mathcal{F}$ finds the optimal solution on $(X,d')$, and so the clustering it selects can only have lower cost than the cost of $C$ when evaluated with $d'$. 

So, we calculate the cost of $C$ on $(X,d')$. The $k$-means cost function is bounded by $\sum_{i=1}^k \sum_{x,y \in C_i} \frac{1}{|C_i|} d(x,y)^2.$ After a multiplicative perturbation, the contribution of an edge of length $d(x, y)$, which used to contribute at most $d(x, y)^2$ to the cost of the function, contributes at most $(\alpha \cdot d(x, y))^2$, and so the contribution increases by at most a factor of $\alpha ^2$. $cost_{d'}(\mathcal{F}(X,d')) \leq cost_{d'}(C) \leq \sum_i ^k \sum _{\{x, y\} \subseteq C_i} \frac{1}{|C_i|} (\alpha \cdot d(x, y) )^2 \leq \alpha ^2 cost_{d}(C)$. 

Then, $c\geq \alpha ^2$, so $\alpha \leq \sqrt{c}$.
\end{proof}

The proofs for $k$-medoid and min-sum follow similarly and are included in the appendix.

\subsection{Perturbation robustness of Linkage-Based algorithms}

We now move onto Linkage-Based algorithms, which in contrast to the methods studied in the previous section, do not seek to optimize an explicit objective function. Instead, they perform a series of merges, combining clusters according to their own measure of between-cluster distance. 

Given clusters $A, B \subseteq X$, the following are the between-cluster distances of some of the most popular Linkage-Based algorithms:
\begin{itemize}
\item \textbf{Single linkage:} $\min_{a \in A, b \in B} d(a,b)$
\item \textbf{Average linkage:} $\sum_{a\in A, b \in B} \frac{d(a,b)}{(|A|\cdot|B|)}$
\item \textbf{Complete linkage:} $\max_{a \in A, b \in B} d(a,b)$
\end{itemize}

We consider Linkage-Based algorithms with the $k$-stopping criterion, which terminate an algorithm when $k$ clusters remain, and return the resulting partitioning.

%Because no explicit objective functions are used, we cannot rely on the uniqueness of optimum notion of clusterability. Instead we shift our focus to strict separation, which while related to uniqueness of optimum is not specific to objective functions and more closely relates to how linkage based clustering functions view between-cluster distance. We show in the appendix that the clustering functions in the prior section are not perturbation robust on strictly separable data sets.

\begin{comment}
We begin our study of perturbation robust data by considering cluster structures required for perturbation robustness of Linkage-Based algorithms. These algorithms use a greedy approach; at first every element is in its own cluster. Then the algorithm repeatedly merges the ``closest'' pair of clusters until some stopping criterion is met.
\end{comment}

%Observe that the linkage functions of the most common Linkage-Based algorithms, single-linkage, average-linkage, and complete-linkage, all satisfy the above condition. 

Because no explicit objective functions are used, we cannot rely on the uniqueness of optimum notion of clusterability. To define the type of cluster structure on which Linkage-Based algorithms exhibit perturbation robustness, we introduce a natural measure of clusterability based on a definition by Balcan et al~\cite{blum}. The original notion required data to contain a clustering where every element is closer to all elements in its cluster than to all other points. This notion was also used in ~\cite{ackerman2011weighted}, \cite{Reyzin}, and \cite{ackerman2014incremental}. See Figure~\ref{fig:nice} for an illustration. 

%We present a continuation of this idea, incorporating into it the concept of perturbations.

\begin{definition}[$(\alpha,k)$-strictly separable]
A data set $(X,d)$ is $(\alpha,k)$-\emph{strictly separable} if there exists a unique clustering $C = \{C_1, \ldots, C_k\}$ of $X$ so that for all $i \neq j$ and all $x,y \in C_i$, $z \in C_j$, $\alpha d(x, y) \leq d(x, z)$. 
\end{definition}

The definition for $(\epsilon,k)$-strictly additive separable is analogous.

\begin{definition}[$(\epsilon,k)$-Strictly Additive Separable]
A data set $(X,d)$ is $(\epsilon,k)$-\emph{Strictly Additive Separable} if there exists a unique clustering $C = \{C_1, \ldots, C_k\}$ of $X$ so that for all $i \neq j$ and all $x,y \in C_i$, $z \in C_j$, $ d(x, y) + \epsilon \leq d(x, z)$. 
\end{definition}

Before moving on to our results for Linkage-Based algorithms, we show that the above notions of clusterability are not sufficient to show that data is perturbation robust for $k$-means and similar methods. This indicates that different algorithms require different cluster structures in order to exhibit perturbation robustness. We show this results for $(\alpha, k)$-strictly separable data. The proof for $(\epsilon,k)$-strictly additive separable data is in the supplementary material.

%Let $\mathcal{F}$ be any one of $k$-means, $k$-medoids, or min-sum. Then for any $\alpha >1, \: \delta < \frac{2(k-1)n}{k^2 (n-1)}$ or $\epsilon>0 , \: \delta < \frac{(k-2)(n-k)^2}{(k-1)n^2}$, there exists an $(\alpha, k)$-strictly separable or $(\epsilon,k)$-strictly-additive separable data set on which $\mathcal{F}$ is not $(\alpha, k)$-multiplicative perturbation robust or $(\epsilon, \delta)$-additive perturbation robust respectively.

\begin{figure}
\begin{center}
\includegraphics[scale=0.4]{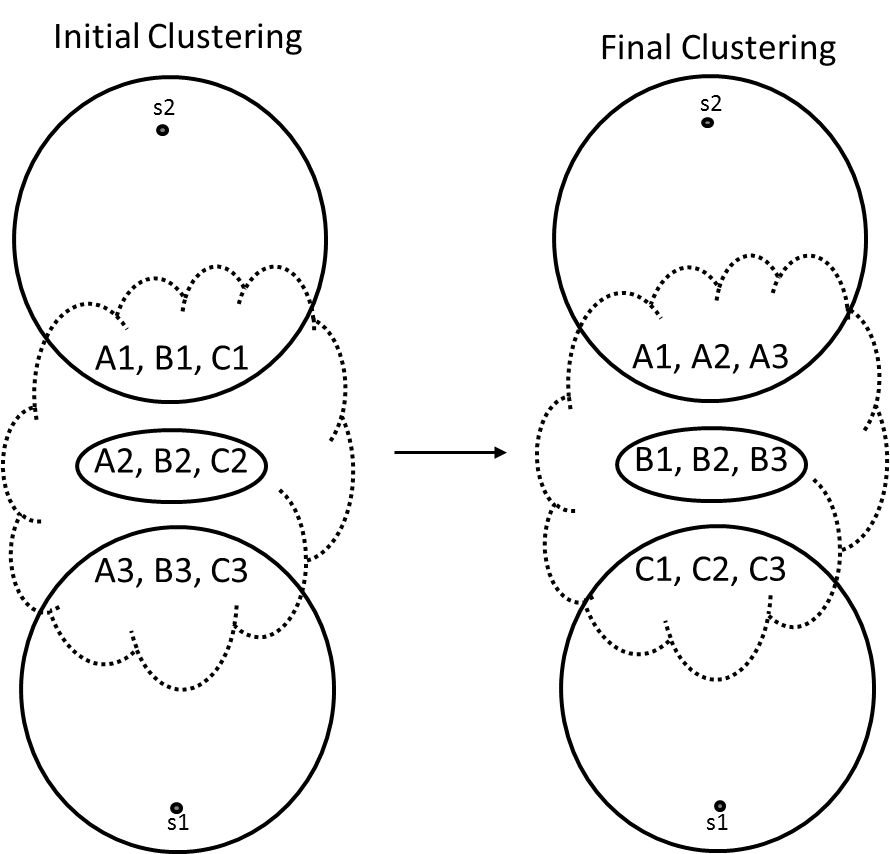}
\end{center}
\caption{An illustration of the data set in the proof of Theorem~\ref{bigthm} for $k=3$. The original data consists of a dense region, ``cloud'' of points, and $k-1$ outliers. Each of the $A_i$s, $B_i$s and $C_i$s consist of $\frac{n}{k^2}$ points. Before the perturbation, the data in the cloud is clustered $\{\{A_1\cup B_1\cup C_1\}, \{A_2\cup B_2\cup C_2\}, \{A_3\cup B_3\cup C_3\}\} $, whereas after  the perturbation it is partitioned as  $\{\{A_1\cup A_2\cup A_3\}, \{B_1\cup B_2\cup B_3\}, \{C_1\cup C_2\cup C_3\}\} $, leading to a large Hamming distance.}
%of the  re-clustering of of points to maximize the Hamming distance produced by an $\alpha$-multiplicative perturbation of an $(\alpha, k)$-strictly separable data set when $k=3$. For this figure, groups of $n/k^2$ points are referred to by capital letters and a number and singleton points are denoted by a lowercase s and a number. A solid line denotes the optimal clusterings and the dotted cloud denotes the groups of points in the cloud.
\label{fig:multGroup}
\end{figure}

\begin{theorem}\label{bigthm}
Let $\mathcal{F}$ be any one of $k$-means, $k$-medoids, or min-sum. Then for any $\alpha >1, \: \delta < \frac{2(k-1)n}{k^2 (n-1)}$ , there exists an $(\alpha, k)$-strictly separable data set on which $\mathcal{F}$ is not $(\alpha, \delta)$-multiplicative perturbation robust.
\end{theorem}

\begin{proof}

We construct such a data set $(X,d)$ such that there is one cloud of points densely packed together and  $k-1$ singleton points far away from all other points. Further, in this construction $n|k^2$ and $n>>k$.  The strictly separable clustering of $X$ will consist of the $k-1$ singleton points being separate clusters and  the cloud being in one cluster. 

Arrange the cloud of points such that the ratio of the largest to smallest in-cloud distance is less than $\alpha$, the ratio of the largest to smallest cloud point to singleton point distance is less than $\alpha$, and all points are separated such that $\mathcal{F}$ splits the cloud evenly into $k$ separate clusters and the singleton points go into separate clusters.

Because the ratio between the largest and smallest in-cloud distance is less than $\alpha$, an $\alpha$-multiplicative perturbation can radically change the structure of the points in the cloud. 
If a point is identified by its distances to all other data, then perturbing the distance function can cause points to switch with one another. This ability to change applies similarly to the distance between cloud and singleton points. Because points can be arbitrarily made to act like other points, we perturb the data set such that the maximum number of in/between-cluster relationships are changed (with the restriction of never switching points from being in the cloud to being a singleton point and vice versa, because a multiplicative perturbation cannot necessarily make this switch).

%To approximate the Hamming distance between the pre-perturbed and post-perturbed clustering we find the Hamming distance only taking into account the points in the cloud and make the simplifying assumptions that $k^2$ divides $n$ and $n/k > 1$. When only considering the points in the cloud the maximum Hamming distance between the two clusterings will be $\frac{2(k-1)n}{k^2 (n-1)}$.

We maximize the possible Hamming distance under the previous assumptions by constructing the following two data sets: First, divide the points of the cloud evenly into $k$ clusters. This is our first clustering. Then taking that clustering, re-cluster the points by grouping the points in each cluster into groups of $n/k^2$. Finally, form the new clusters by selecting one group from each previous cluster to be in a new cluster. See figure ~\ref{fig:multGroup} for an example.

To find the Hamming distance between these two clusterings we first find the number of pairwise relationships that were formerly between-cluster that are now in-cluster. First, remember that a group contains $n/k^2$ points and there will be $k  \binom{k}{2}$ group pairs that were formerly between-cluster that are now in cluster. Next, each point in a group will contribute $n/k^2$ to the Hamming distance per group pair. This gives the amount contributed to the Hamming distance by points relationships that were formerly between-cluster that are now in-cluster as $k  \frac{n^2}{k^4}  \binom{k}{2}$.

We now find the number of pairwise dissimilarities that were formerly in-cluster that are now between-cluster. Similar to before, each group contains $n/k^2$ points and there will be $k \binom{k}{2}$ groups that were formerly in-cluster that are now between-cluster. This gives the total Hamming distance as $2k \frac{n^2}{k^4}\binom{k}{2}\binom{n}{2}^{-1}$, which reduces to $\frac{2(k-1)n}{k^2 (n-1)}$.

%This gives the intuitive result that the amount contributed to the Hamming distance by the number of pairwise relationships that went from in-cluster to between-cluster is the same as the amount that went from between-cluster to in-cluster. 

\end{proof}
%See the appendix for the proof for additive perturbation. 

\begin{figure}
\begin{center}
\includegraphics[scale=0.32]{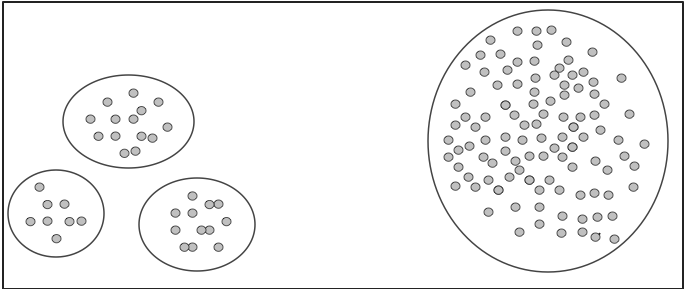}
\end{center}
\caption{An example of strictly separable data. Note that it may include clusters with different diameters, as long as the dissimilarity between any two clusters scales as the larger diameter of the two.}
\label{fig:nice}
\end{figure}

We now show that whenever data is strictly separable, then it is also perturbation robust with respect to some of the most popular Linkage-Based algorithms. An analogous result for additive perturbation robustness appears in the supplementary material.

\begin{theorem}
Single-Linkage, Average-Linkage, and Complete-Linkage are $(\alpha,0)$-multiplicative perturbation robust on all $(\alpha ^2 ,k)$-strictly separable  data sets. 
\end{theorem}
\begin{proof}

We begin by showing that whenever data is $(1,k)$-strictly separable, then these Linkage-Based algorithms identify the underlying cluster structure. This result was previously shown for Single-Linkage~\cite{blum} and Average-Linkage~\cite{ackerman2011weighted}. We now prove this for complete-linkage. 

First, we introduce the concept of a \emph{refinement}. A clustering $C'$ is a \emph{refinement} of clustering $C^*$ if $C^*$ can be obtained by merging clusters in $C'$. The proof proceeds by induction on the number of iterations, showing that at each step of the algorithm, the current clustering is a refinement of the strictly separable $k$-clustering $C$. Since Linkage-Based algorithms start by placing each point in its own cluster, the clustering formed in the first step is a refinement of $C$. Assuming that the hypothesis holds at step $i$ of the algorithm, we show that it is retained  in the following step. Consider any $C_1$ and $C_2$ that are a subset of the same cluster in $C$, and any $C_3$ that is a subset of a different cluster in $C$. Let $(x,y) = argmax_{x\in C_1, y\in C_2}d(x,y)$. Then, the dissimilarity between $x$ and any point in $C_3$ is greater than $d(x,y)$ since the data is $(1,k)$-strictly separable. Then Complete-Linkage merges $C_1$ with $C_2$ before merging $C_1$ with $C_3$. 

Lastly, observe that data that is $(\alpha ^2 ,k)$-strictly separable is also  $(1,k)$-strictly separable, and remains so after an $\alpha$-perturbation, as shown in  Lemma 1 in the appendix. It follows that single, average, and complete linkage are $(\alpha,0)$-Multiplicative Perturbation Robust on $(\alpha ^2,k)$-strictly separable  data. 
\end{proof}

%Additive perturbation robustness follows similarly and is included in the supplementary material.

\begin{comment}
\begin{lemma}\label{lemma_mult}
The minimum separation a data set that is  $(\alpha ^2, k)$-strictly separable can be after an $\alpha$-multiplicative perturbation is $(1, k)$-strictly separable. 
\end{lemma}
\begin{proof}
If a data set is $(\alpha ^2, k)$-strictly separable, the minimum between-cluster dissimilarity must be greater than the maximum in-cluster dissimilarity by a factor of $\alpha ^2$, $argmin(d(x, y)|x, y \in X x \not \sim _C y)\geq\alpha ^2 argmax(d(x, y)|x, y \in X x \sim_C y)$.

An $\alpha$-multiplicative perturbation can change any pairwise distance by at most a factor of $\alpha ^{\pm 1}$, $argmin(\alpha \cdot d(x, y)|x, y \in X x \not \sim _C y) \geq \alpha ^2 argmax( \alpha ^{-1} \cdot d(x, y)|x, y \in X x \sim_C y)$ because the multiplication is applied to all arguments and is outside of the distance function this simplifies to $\alpha \cdot argmin(d(x, y)|x, y \in X x \not \sim _C y) \geq \alpha \cdot argmax(d(x, y)|x, y \in X x \sim_C y)$. 

Therefore $argmin(d(x, y)|x, y \in X x \not \sim _C y) \geq argmax(d(x, y)|x, y \in X x \sim_C y)$. This is equivalent to the definition of $(1, k)$-strictly separable.
\end{proof}
\end{comment}

\vspace{-3mm}
\section{Conclusions}
\vspace{-2mm}

As a property of an algorithm, perturbation robustness fails in a strong sense, contradicting even more fundamental requirements of clustering functions. As such, no  algorithm can exhibit this desirable characteristic on all data sets. Notably, this result persists even if we allow two-thirds of all pairwise distance to change following a perturbation. 

However, a more optimistic picture emerges when considering clusterable data, and we show that popular  paradigms are able to discover some cluster structures even on faulty data. Further, different clustering techniques are perturbation robust on different cluster structures. This has important implications for the ``user's dilemma,'' which is the problem of selecting a suitable clustering algorithm for a given task. Faced with the challenge of clustering data with imprecise dissimilarities between pairwise entities, a user cannot simply elect to apply a perturbation robust technique as no such methods exist, and as such the selection of suitable methods calls for some insight on the underlying structure of the data. 

Future work will investigate robustness of heuristics, such as Lloyd's method, for which preliminary analysis suggests that the cluster structure required for perturbation robustness depends on the method of initialization.

%This phenomenon persists even if we allow some structural changes to the output of the algorithm after the data is perturbed. As such, it appears necessary to consider the cluster structure of the underlying data to evaluate an algorithm's perturbation robustness. 

%Our results reveal that different algorithms require different types of cluster structure in order to attain perturbation robustness. While linkage-based techniques are robust on data that consists of well-separable clusters (as captured by Strict Separation~\cite{blum}), this condition does not guarantee any degree of perturbation robustness for methods such as $k$-means, $k$-medoids, or min-sum. The latter techniques achieve robustness when data has a unique optimum, in the sense that any clustering with near-optimal cost is also structurally similar to the optimal solution (a notion studied by~\cite{Balcan}). 

% Furthermore, to connect theory with practice, it would be interesting to explore methods for empirically estimating the perturbation robustness of a given data set.

\newpage
\part*{Appendix}
\section{Impossibility theorems for $\delta=0$}~\label{first impossibility}

In prior work, perturbation robustness is often defined with $\delta=0$, requiring that the clustering remain unchanged after a perturbation. Unfortunately, as a property of an algorithm, this formulation fails in a strong sense. That is, we show that no reasonable clustering function can satisfy this condition without effectively ignoring all pairwise distances, by outputting the same partitioning irrespective of the setting of $d$.

%\begin{definition}[$\alpha$-Multiplicative Perturbation Robust Clustering Function]\label{multiplicativerobustness}
%A clustering function $\mathcal{F}$ is \emph{$\alpha$-Multiplicative-Perturbation Robust}, if for any $d'$ that is an $\alpha$-multiplicative-perturbation of $d$, $\mathcal{F}(X,d) = \mathcal{F}(X,d').$
%\end{definition}

%\begin{definition}[$\epsilon$-Additive Perturbation Robust Clustering Function]\label{additiverobustness}
%A clustering function $\mathcal{F}$ is \emph{$\epsilon$-Additive-Perturbation Robust}, if for any $d'$ that is an $\epsilon$-additive-perturbation of $d$, $\mathcal{F}(X,d) = \mathcal{F}(X,d').$
%\end{definition}

Specifically, we prove that both additive and multiplicative perturbation robustness with $\delta=0$ contradicts $2$-Richness, which is a relaxation  Kleinberg's Richness axiom. This property requires clustering functions to be at least minimally responsive to pairwise dissimilarity. That is, given complete freedom to reassign all dissimilarities, we should be able to change the output of the function. This basic property is satisfied by all reasonable clustering methods~\cite{NIPS2010}.

Let $Range(\mathcal{F},X, k)$ denote the set of all clusterings $C$ so that $\mathcal{F}(X,d, k) = C$ for some dissimilarity function $d$.

\begin{definition}[2-Richness]
For all $X$, $|Range(\mathcal{F},X, k)|\geq 2.$
\end{definition}
We now prove that both additive and multiplicative perturbation robustness (with $\delta=0$) are  inconsistent with $2$-Richness.

%\begin{theorem}
%Given any $\epsilon>0$ and $\delta< \frac{(k-1)(k-2)n}{k^2(n-1)}$, there is no $k$-clustering function that simultaneously satisfies Scale-Invariance, K-Richness, and $(\epsilon, \delta)$-Perturbation Robustness. 
%\end{theorem}
%\begin{proof}
%By way of contradiction, let $\mathcal{F}$ be such a function. Let $X$ be a domain set on $n$ points where $n$ is divisible by $k$. By K-Richness, there exists a dissimilarity function $d$ so that $\mathcal{F}(X,d,k)$ has all but $k-1$ of the points in the same cluster. Also by K-Richness there also exists a dissimilarity function $d'$ so that $\mathcal{F}(X,d',k)$  divides the data into $k$ equal sized clusters, with all the singletons in the previous clustering belong to the same cluster in this partition. 

%Observe that the clusterings on these data sets differ by at least $\frac{\binom{k-1}{2}(n/k)^2}{(\binom{n}{2}} = \frac{(k-1)(k-2)n}{k^2(n-1)}$. Now, scale both data sets, so that all pairwise dissimilarities are smaller than $\epsilon$. Now, there is an $\epsilon$-perturbation for both data sets that makes them isomorphic. Since the algorithm is $\epsilon$-Perturbation-Robust, this contradicts isomorphism invariance, required for clustering functions. 
%\end{proof}

\begin{theorem}\label{impossible_additive}
No clustering function is both 2-Rich and $\epsilon$-Additive Perturbation Robust for any $\epsilon>0$. 
\end{theorem}
\begin{proof}
Let $\mathcal{F}$ be any $2$-Rich clustering function.  Then for any domain set $X$, there exist dissimilarity functions $d$ and $d'$ so that $\mathcal{F}(X,d,k) \neq \mathcal{F}(X,d',k)$. Observe that we can transform $d$ into $d'$ by making incremental changes, each changing the dissimilarity function by no more than $\epsilon$ on each pairwise dissimilarity. Then, by  $\epsilon$-Perturbation-Robustness,  $\mathcal{F}(X,d,k) = \mathcal{F}(X,d',k)$, contradicting the previous claim.  
\end{proof}

Now we prove the analogous result for multiplicative perturbation robustness. 
\begin{theorem}\label{impossible_mult}
No clustering function is both $2$-Rich and $\alpha$-Multiplicative Perturbation Robust for any $\alpha>1$. 
\end{theorem}
\begin{proof}
Let $\mathcal{F}$ be any $2$-Rich, Isomorphic Invariant, and $\alpha$-Multiplicative Perturbation Robust clustering function.  Then for any domain set $X$, there exist dissimilarity functions $d$ and $d'$ so that $\mathcal{F}(X,d,k) \neq \mathcal{F}(X,d',k)$. Observe we can transform $d$ into $d'$ through a series of $\alpha$-multiplicative perturbations, changing each pairwise distance by a factor of $\alpha$ or less with each perturbation. 

A contradiction is therefore achieved. By $\alpha$-Multiplicative Perturbation Robustness $\mathcal{F}(X,d,k) = \mathcal{F}(X,d',k)$, and by $2$-Richness $\mathcal{F}(X,d,k) \not = \mathcal{F}(X,d',k)$.
\end{proof}

The above results implies that perturbation robustness with $\delta=0$ is too stringent a requirement for clustering functions.

%Embedding the data into Euclidean space reveals a contrast in the relative strengths of $\alpha$-multiplicative perturbations and $\epsilon$-additive perturbations. Take a data set comprising of two points in Euclidean space $X= \{(1, 0), (-1, 0)\}$. 

%First, apply an arbitrary amount of $\alpha$-multiplicative perturbations to the data set with the intent of switching the places of the two data points. The data points can never be switched no matter the $\alpha$ selected. The closest the two points can ever become to being switched is having their distance approach zero. This is intuitively the same as the fact that no amount of multiplications and divisions of positive numbers onto a positive number can produce a negative result.

%Now, apply an arbitrary amount of $\epsilon$-additive perturbations to the data set. Notice the data points can easily be switched through a series of $\epsilon$-additive perturbations no matter the $\epsilon$ chosen, provided it is greater than zero. This is intuitively the same as the fact that a series a of subtractions of a positive number from a positive number can produce a negative result.

%This means for the above proof to work in Euclidean space the property of Isomorphism Invariance must be added for $\alpha$-multiplicative perturbation but not for $\epsilon$-additive perturbation.

\section{Proof of Lemma 1}

\begin{lemma}\label{lemma_mult}
The minimum separation a data set that is  $(\alpha ^2, k)$-strictly separable can be after an $\alpha$-multiplicative perturbation is $(1, k)$-strictly separable. 
\end{lemma}
\begin{proof}
If a data set is $(\alpha ^2, k)$-strictly separable, the minimum between-cluster dissimilarity must be greater than the maximum in-cluster dissimilarity by a factor of $\alpha ^2$, $argmin(d(x, y)|x, y \in X x \not \sim _C y)\geq\alpha ^2 argmax(d(x, y)|x, y \in X x \sim_C y)$.

An $\alpha$-multiplicative perturbation can change any pairwise distance by at most a factor of $\alpha ^{\pm 1}$, $argmin(\alpha \cdot d(x, y)|x, y \in X x \not \sim _C y) \geq \alpha ^2 argmax( \alpha ^{-1} \cdot d(x, y)|x, y \in X x \sim_C y)$ because the multiplication is applied to all arguments and is outside of the distance function this simplifies to $\alpha \cdot argmin(d(x, y)|x, y \in X x \not \sim _C y) \geq \alpha \cdot argmax(d(x, y)|x, y \in X x \sim_C y)$. 

Therefore $argmin(d(x, y)|x, y \in X x \not \sim _C y) \geq argmax(d(x, y)|x, y \in X x \sim_C y)$. This is equivalent to the definition of $(1, k)$-strictly separable.
\end{proof}

\section{Excluded multiplicative perturbation proofs}

\begin{theorem}
Consider the $k$-medoids clustering function and a data set $(X,d)$. If $(X,d)$ is \emph{$(\delta, c, c_0, k)$-uniquely optimal}, then it is also $(\alpha , \delta,k)$-multiplicative perturbation robust for all  $\alpha < c$.
\end{theorem}

\begin{proof}
Consider a data set $(X,d)$, and let $d'$ be any $\alpha $-multiplicative perturbation of $d$. 

First, we argue that $cost_{d'}(\mathcal{F}(X,d,k))$ is close to $cost_{d'}(\mathcal{F}(X,d',k))$. Let $C = \mathcal{F}(X,d,k)$
First, note that $cost_{d'}(F(X,d',k)) \leq cost_{d'}(C)$. This is because $\mathcal{F}$ finds the optimal solution on $(X,d')$, and so the clustering it selects can only have lower cost than the cost of $C$. We also compute the cost of each element to the same cluster centers as in $C$, as it provides an upper bound on $ cost_{d'}(C$. So, $cost_{d'}(C) \leq \alpha \cdot cost_{d}(C).$ So,  $c \geq \alpha$, so $\alpha < c$.
\end{proof}

\begin{theorem}
 Given the $min\textrm{-}sum$ objective, a data set that is $(\delta, \alpha, 0, k)$-uniquely optimal is also $(\alpha, \delta)$-multiplicative perturbation robust. 
 \end{theorem}
 \begin{proof}
 Let $C$ be the optimal $min\textrm{-}sum$ clustering of an arbitrary data set.

 Let $d'$ be any $\alpha$-multiplicative perturbation of $d$. Therefore $cost_{d'}(C)$ is at worst:
 $$ \Sigma _{i=1}^k \Sigma _{x, y \in C_i} d'(x, y)$$
 $$= \Sigma _{i=1}^k \Sigma _{x, y \in C_i} \alpha d(x, y)$$
 $$=\alpha \cdot cost_d(C)$$
 Because $cost_{d'}(C) \leq \alpha \cdot cost_d(C)$, $\Delta (\mathcal{F}(X,d,k), \mathcal{F}(X, d', k)) < \delta$. Therefore the maximum change an $\alpha$-multiplicative perturbation can produce approaches  $\delta$.
 \end{proof}

\section{Equivalent results for additive perturbation}
We now prove the results in the main paper but for additive instead of multiplicative perturbation.

First, recall the definition of additive perturbation. 

\begin{definition}[$\epsilon$-additive perturbation of a dissimilarity function]\label{additiveperturbationdissimilarityAppendix}
Given a pair of dissimilarity functions $d$ and $d'$ over a domain $X$, $d'$ is an \emph{$\epsilon$-additive perturbation} of $d$, for $\epsilon>0$, if for all $x,y \in X$, $d(x,y)-\epsilon \leq d'(x,y) \leq d(x,y)+\epsilon$.
\end{definition}

\begin{definition}[$(\epsilon, \delta)$-Additive Perturbation Robust Clustering Function]\label{additiverobustness}
A clustering function $\mathcal{F}$ is \emph{$(\epsilon, \delta)$-Additive-Perturbation Robust}, if for any $d'$ that is an $\epsilon$-additive-perturbation of $d$, $\Delta(\mathcal{F}(X,d, k), \mathcal{F}(X,d', k))\leq \delta.$
\end{definition}

\begin{theorem}
Consider the $k$-medoids clustering function. If $(X,d)$ is \emph{$(\delta, c, c_0, k)$-uniquely optimal}, then it is also $(\epsilon , \delta,k)$-additive perturbation robust for all  $\epsilon <\sqrt{\frac{2c_0}{n(n-1)}}$.
\end{theorem}

\begin{proof}
Consider a data set $(X,d)$, and let $d'$ be any $\epsilon $-additive perturbation of $d$. 

First, we argue that $cost_{d'}(\mathcal{F}(X,d,k))$ is close to $cost_{d'}(\mathcal{F}(X,d',k))$. Let $C = \mathcal{F}(X,d,k)$
First, note that $cost_{d'}(\mathcal{F}(X,d',k)) \leq cost_{d'}(C)$. This is because $\mathcal{F}$ finds the optimal solution on $(X,d')$, and so the clustering it selects can only have lower cost than the cost of $C$. We also compute the cost of each element to the same cluster centers as in $C$, as it provides an upper bound on $ cost_{d'}(C)$. So, $cost_{d'}(C) \leq cost_{d}(C) +\epsilon \binom{n}{2}.$	So,  $c_0 \geq \epsilon \binom{n}{2}$, so $\epsilon <\sqrt{\frac{2c_0}{n(n-1)}}$. 
\end{proof}

 \begin{theorem}
 Given the $min\textrm{-}sum$ objective, $(\delta, 0, \epsilon \binom{n}{2}, k)$-uniquely optimal data sets are $(\epsilon, \delta)$-additive perturbation robust.
 \end{theorem}
 \begin{proof}
 Let $C$ be the optimal $k$-means clustering of an an arbitrary data set.

 Let $d'$ be any $\epsilon$-perturbation of $d$. Therefore $cost_{d'}(C)$ is at most:

 $$\sum\limits _{i=1} ^k \sum\limits _{x, y \in C_i} d'(x, y) $$
 $$ \leq \sum\limits _{i=1} ^k \sum\limits _{x, y \in C_i} d(x, y)+\epsilon$$
 $$\leq cost_{d}(C)+\sum\limits _{i=1} ^k \sum\limits _{x, y \in C_i} \epsilon$$
 $$\leq cost_d(C) + \epsilon \binom{n}{2}$$
 Because $cost_{d'}(C) < cost_d(C) + \epsilon \binom{n}{2}$, $\Delta ( \mathcal{F}(X, d, k), \mathcal{F}(X, d', k)) < \delta$. Therefore the maximum change an $\epsilon$-perturbation can produce approaches $\delta$. 
\end{proof}

\begin{definition}[$(\epsilon,k)$-Strictly Additive Separable]
A data set $(X,d)$ is $(\epsilon,k)$-\emph{Strictly Additive Separable} if there exists a unique clustering $C = \{C_1, \ldots, C_k\}$ of $X$ so that for all $i \neq j$ and all $x,y \in C_i$, $z \in C_j$, $ d(x, y) + \epsilon \leq d(x, z)$. 
\end{definition}

\begin{theorem}
Let $\mathcal{F}$ be any one of $k$-means, $k$-medoids, or min-sum. Then for any $\epsilon>0 , \: \delta < \frac{(k-2)(n-k)^2}{(k-1)n^2}$, there exists an $(\epsilon,k)$-strictly-additive separable data set on which $\mathcal{F}$ is not  $(\epsilon, \delta)$-additive perturbation robust.\end{theorem}
\begin{proof}
With respect to $\epsilon$-additive perturbations:

Taking the same base data set as before we now set the  cloud to be sufficiently dense, and the singleton points to be organized so that $\mathcal{F}$  merges all but the large cluster, and as such the final clustering depends nearly entirely on the internal structure of the large cluster. 

Arrange the the data set to consist of dissimilarities smaller than $\epsilon$, so an $\epsilon$-perturbation of $d$ can radically alter the output of $\mathcal{F}$. In particular, we can arrange the large cluster so that $\mathcal{F}(X,d,k)$ subdivides it into $k$ equal size groups, and after the perturbation all points in that cluster form a single cluster by moving the points in that cluster very close together. 

We now compute the distance between these two clusterings. The original clustering, looking only at the data in the large cluster, has $\binom{k-1}{2} (\frac{n-k+1}{2})^2$ between-cluster edges, all of which become incluster edges after the perturbation. So, the two clusterings differ by at least $\binom{k-1}{2}(\frac{n-k+1}{k-1})^2/\binom{n}{2} > \frac{(k-2)(n-k)^2}{(k-1)n^2}.$
\end{proof}
\begin{theorem}
Single-Linkage, Average-Linkage, and Complete-Linkage are $(\epsilon,0)$-perturbation robust on all $(2\epsilon,k)$-strictly additive Separable data sets. 
\end{theorem}
\begin{proof}
We begin by showing that whenever data is $(0,k)$-strictly additive separable, then these Linkage-Based algorithms identify the underlying cluster structure. This result was previously shown for Single-Linkage and Average-Linkage. We now prove this for complete-linkage. 

Recall the definition of a \emph{refinement}. A clustering $C'$ is a \emph{refinement} of clustering $C^*$ if $C^*$ can be obtained by merging clusters in $C'$. The proof proceeds by induction on the number of iterations, showing that at each step of the algorithm, the current clustering is a refinement of the strictly separable $k$-clustering $C$. Since Linkage-Based algorithms start by placing each point in its own cluster, the clustering formed in the first step is a refinement of $C$. Assuming that the hypothesis holds at step $i$ of the algorithm, we show that it is retained  in the following step. Consider any $C_1$ and $C_2$ that are a subset of the same cluster in $C$, and any $C_3$ that is a subset of a different cluster in $C$. Let $(x,y) = argmax_{x\in C_1, y\in C_2}d(x,y)$. Then, the dissimilarity between $x$ and any point in $C_3$ is greater than $d(x,y)$ since the data is $(0,k)$-strictly separable. Then Complete-Linkage merges $C_1$ with $C_2$ before merging $C_1$ with $C_3$. 

Lastly, observe that $(2 \epsilon,k)$-strictly additive separable data is also  $(0,k)$-strictly additive separable, and remains so after an $\epsilon$-perturbation. It follows that Single-Linkage, Average-Linkage, and Complete-Linkage are $(\epsilon,0)$-Perturbation-Robust on $(2\epsilon,k)$-strictly separable  data. See lemma below.
\end{proof}

\begin{lemma}
 The minimum separation a data set that is  $(2 \epsilon, k)$-strictly additive separable can be after an $\epsilon$-additive perturbation is $(0, k)$-strictly additive separable
\end{lemma}
\begin{proof}
If a data set is $(2 \epsilon, k)$-strictly additive separable then the maximum in cluster distance is smaller than the minimum out of cluster distance by $2 \epsilon$,
$argmin(d(x, y) | x, y \in X  \ x \not \sim _c y) - argmax(d(x, y) | x, y \in X  \ x \sim _c y) = 2 \epsilon$.
An $\epsilon$-additive perturbation can change any pairwise distance by at most $\pm \epsilon$,
$argmin(d(x, y) + \epsilon | x, y \in X  \ x \not \sim _c y) - argmax(d(x, y) - \epsilon | x, y \in X  \ x \sim _c y) = 2\epsilon$.
because addition and subtraction is applied to all arguments this simplifies to 
$argmin(d(x, y) | x, y \in X  \ x \not \sim _c y) + \epsilon - (argmax(d(x, y) | x, y \in X  \ x \sim _c y) - \epsilon) =2 \epsilon$.
Therefore,
$argmin(argmin(d(x, y) | x, y \in X  \ x \not \sim _c y) - argmax(d(x, y) \pm \frac{\epsilon}{r} | x, y \in X  \ x \sim _c y)) = 0$.
This is equivalent to being $(0, k)$-strictly additive separable. 
\end{proof}

%MIN SUM Let $C = F(X,d,k)$
%First, note that $cost(F(X,d',k)) \leq cost(C, (X,d'))$. This is because $F$ finds the optimal solution on $(X,d')$, and so the clustering it selects can only have lower cost than the cost of $C$. 
%So, we calculate the cost of $C$ on $(X,d')$.  $cost(C, (X,d')) \leq cost(C, (X,d'))+\epsilon \binom{n}{2}$.	So,  $c_0 \geq \epsilon^2\binom{n}{2}$, so $\epsilon<\sqrt{\frac{2c_0}{n(n-1)}}$. 

%k-medoids
%First, we argue that $cost(F(X,d,k))$ is close to $cost(F(X,d',k))$. Let $C = F(X,d,k)$
%First, note that $cost(F(X,d',k)) \leq cost(C, (X,d'))$. This is because $F$ finds the optimal solution on $(X,d')$, and so the clustering it selects can only have lower cost than the cost of $C$. We also compute the cost of each %element to the same cluster centers as in $C$, as it provides an upper bound on $ cost(C, (X,d'))$. So, $cost(C, (X,d')) \leq cost(C, (X,d')) +\binom{n}{2}\epsilon.$ .	So,  $c_0 \geq \epsilon^2\binom{n}{2}$, so $\epsilon<\sqrt{\frac{2c_0}{n(n-1)}}$. 

\section{Impossibility Result for Euclidean Space Clustering}

\begin{definition}[Euclidean $\alpha$-multiplicative perturbation]

Given a data set embedded into Euclidean space $X$, an $\alpha$-multiplicative perturbation of $X$ produces $X'$ s.t. all $x \in X$ can be moved freely in the space to produce $x'$ given that for all $x, y \in X, \: x', y' \in X'$, $\alpha ^{-1} \cdot d(x, y) \leq d(x', y') \leq \alpha \cdot d(x', y')$.
\end{definition}

\begin{definition}[Euclidean $\epsilon$-additive perturbation]

Give a data set embedded into Euclidean space $X$, an $\epsilon$-additive pertubation of $X$ produces $X'$ s.t. all $x \in X$ can be moved freely in the space to produce $x'$ given that for all $x, y \in X, \: x', y' \in X'$, $d(x, y) - \epsilon \leq d(x', y') \leq d(x', y') + \epsilon$.
\end{definition}

\begin{theorem}
There can be no clustering algorithm that follows the three body rule, is data replication invariant and is either $\alpha$-multiplicative perturbation robust or $\epsilon$-additive perturbation robust while clustering data in Euclidean space.
\end{theorem}
\begin{proof}
Like before we create a data set $X = \{a, b, c \}$. However this time we define where the points lie in Euclidean space. Let $a=\vec{1}_d$ Where $\overrightarrow{1_d}$ is the $d$-dimensional vector with all ones. Next we define $b = \overrightarrow{2 + \epsilon '_d}$ for some arbitrarily small $\epsilon '$ and $c = \overrightarrow{3 + \epsilon '_d \:}$. By the three body rule the resulting clustering must be $C = \{\{a\}, \{b, c\}\}$. Now, we perturb the points to create $X'$. The $\alpha$-multiplicative perturbation moves the point $b$ to $\overrightarrow{\frac{2 + \epsilon '}{\alpha}_d}$, and the $\epsilon$-additive perturbation  moves the point $b$ to $\overrightarrow{(2 + \epsilon ' - \epsilon)_d}$ both leaving all other points untouched. Neither of the perturbations change any pairwise distance by a factor of $\alpha ^{\pm 1}$ or in absolute $\pm \epsilon$ respectively.

Note in the multiplicative case the $\epsilon '$ addition can be set arbitrarily small s.t. the  $\alpha ^{-1}$ factor will dominate whether the point is further or closer to the $\overrightarrow{0_d}$ than $\overrightarrow{2_d}$.

By the three body rule $X'$ must be clustered as $C = \{\{a, b\}, \{c\}\}$ and the perturbation robustness demands the clustering be $C = \{\{a\}, \{b, c\}\}$. Finally, data replication invariance can be used to increase the size of $X'$ to be arbitrarily large no longer requiring there to be only three points while still maintaining the contradiction.
\end{proof}

\newpage
\bibliographystyle{abbrvnat}
\bibliography{clustering}
\end{document}